\documentclass{article}

\usepackage{microtype}
\usepackage{subfigure}
\usepackage{booktabs} 
\usepackage[dvips]{graphicx}
\usepackage{amssymb,amsmath,color}
\usepackage{url}

\usepackage{tabularx}
\usepackage{booktabs}

\usepackage{todonotes}
\usepackage{algorithm2e}

\usepackage{stackengine}
\usepackage{hyperref}

\title{Differentially Private Empirical Risk Minimization with Sparsity-Inducing Norms}

\begin{document}

\author{K S Sesh Kumar \\
Department of Computing\\ Imperial College London\\
\texttt{s.karri@imperial.ac.uk}
\and Marc Peter Deisenroth\\
Department of Computing\\ Imperial College London \\
\texttt{m.deisenroth@imperial.ac.uk}
}

\def\delequal{\mathrel{\ensurestackMath{\stackon[1pt]{=}{\scriptstyle\Delta}}}}

\newcommand{\note}[1]{{\textbf{\color{red}#1}}}

\def\ci{\perp\!\!\!\perp} 
\newcommand\independent{\protect\mathpalette{\protect\independenT}{\perp}} 
\def\independenT#1#2{\mathrel{\rlap{$#1#2$}\mkern2mu{#1#2}}} 

\newcommand{\fix}{\marginpar{FIX}}
\newcommand{\new}{\marginpar{NEW}}

\newcommand{\BEAS}{\begin{eqnarray*}}
\newcommand{\EEAS}{\end{eqnarray*}}
\newcommand{\BEA}{\begin{eqnarray}}
\newcommand{\EEA}{\end{eqnarray}}
\newcommand{\BEQ}{\begin{equation}}
\newcommand{\EEQ}{\end{equation}}
\newcommand{\BIT}{\begin{itemize}}
\newcommand{\EIT}{\end{itemize}}
\newcommand{\BNUM}{\begin{enumerate}}
\newcommand{\ENUM}{\end{enumerate}}
\newcommand{\BA}{\begin{array}}
\newcommand{\EA}{\end{array}}
\newcommand{\diag}{\mathop{\rm diag}}
\newcommand{\Diag}{\mathop{\rm Diag}}

\newcommand{\Cov}{{\mathop { \rm cov{}}}}
\newcommand{\cov}{{\mathop {\rm cov{}}}}
\newcommand{\corr}{{\mathop { \rm corr{}}}}
\newcommand{\Var}{\mathop{ \rm var{}}}
\newcommand{\argmin}{\mathop{\rm argmin}}
\newcommand{\argmax}{\mathop{\rm argmax}}
\newcommand{\var}{\mathop{ \rm var}}
\newcommand{\Tr}{\mathop{ \rm tr}}
\newcommand{\tr}{\mathop{ \rm tr}}
\newcommand{\sign}{\mathop{ \rm sign}}
\newcommand{\idm}{I}
\newcommand{\rb}{\mathbb{R}}
\newcommand{\BlackBox}{\rule{1.5ex}{1.5ex}}  
\newcommand{\lova}{Lov\'asz }

\newenvironment{proof}{\par\noindent{\bf Proof\ }}{\hfill\BlackBox\\[2mm]}
\newtheorem{lemma}{Lemma}
\newtheorem{theorem}{Theorem}
\newtheorem{proposition}{Proposition}
\newtheorem{definition}{Definition}
\newtheorem{corollary}{Corollary}

\renewcommand{\labelitemiii}{$-$}

\newcommand{\mysec}[1]{Section~\ref{sec:#1}}
\newcommand{\eq}[1]{Eq.~(\ref{eq:#1})}
\newcommand{\myfig}[1]{Figure~\ref{fig:#1}}
\newcommand{\mytable}[1]{Table~\ref{table:#1}}

\def \supp{ { \rm Supp }}
\def \card{ { \rm Card }}

\allowdisplaybreaks

\bibliographystyle{abbrv}

\maketitle

\begin{abstract}
Differential privacy is concerned about the prediction quality while measuring the privacy impact on individuals whose information is contained in the data.
We consider differentially private risk minimization problems with regularizers that induce structured sparsity.
These regularizers are known to be convex but they are often non-differentiable. We analyze the standard differentially private algorithms, such as output perturbation, Frank-Wolfe and objective perturbation. Output perturbation is a differentially private algorithm that is known to perform well for minimizing risks that are strongly convex. 
%
Previous works have derived excess risk bounds that are independent of the dimensionality. 
In this paper, we assume a particular class of convex but non-smooth regularizers that induce structured sparsity and loss functions for generalized linear models. We also consider differentially private Frank-Wolfe algorithms to optimize the dual of the risk minimization problem. We derive excess risk bounds for both these algorithms. Both the bounds depend on the Gaussian width of the unit ball of the dual norm. We also show that objective perturbation of the risk minimization problems is equivalent to the output perturbation of a dual optimization problem. This is the first work that analyzes the dual optimization problems of risk minimization problems in the context of differential privacy.
\end{abstract}

\section{Introduction}
Large amounts of data are at our disposal to be able to build machine learning systems. The learnability of a system highly relies on the data. However, in  recent years there has been a growing concern on the privacy of the data used to build these systems. Differential privacy~\cite{Dwork2014} based algorithms are used to tackle this issue. It defines the notion of privacy for every data point that is used to train a learning system. Intuitively, differential privacy ensures that that two different datasets that differ in only one data point induce similar output distributions. This ensures the privacy of the particular data point. The advantage of privacy of the data point is not free of cost and comes at a price. The goal of this work is to understand the bounds of the price we pay for privacy of the data points in Empirical Risk Minimization (ERM).

Empirical Risk Minimization (ERM) in supervised machine learning~\cite{Vapnik1995}, uses training data that includes the input vector (also referred to as `feature vector') and the desired output to learn a predictor function. This predictor function is then used to predict the output of unseen data in the testing phase. However, the data is picked uniformly at random as we do not have access to the true distribution of the data. The learnability of the predictor function thus learnt is often characterized by generalization risk, i.e., the expectation of the error of the predictor function in the true distribution of the data points. Empirical risk is referred to the loss of the predictor function when the same data is picked uniformly at random. The generalization bound is the bound on the difference of the generalization risk and the empirical risk. Most supervised machine learning algorithms have generalization bounds that depend on the number of data samples. The generalization error reduces as the number of data points increases.

 Most differentially private ERM~\cite{chaudhuri2011} use techniques, such as adding noise to either the predicted function or some intermediate step in the algorithm. This increases the risk of the predictor function. Excess risk is this increase in risk due to addition of random noise. For most differentially private ERM the excess risk is a function of both the number of data points and the length of the input vector. Earlier works~\cite{jain14, talwar2014}  have derived bounds of excess risk with no dependency on the dimensionality of the input vector. This is useful because learnability already depends on the number of data points.

{\em Contribution.} This work focuses on ERM problems that learn using {\em sparsity inducing norms}~\cite{Bach2010} as regularizers. We have three main contributions in this paper. (i) We consider output perturbation for non-smooth regularizers. (ii) We consider a private version of the Frank-Wolfe algorithm to optimize the dual of the ERM. We provide utility guarantees for these two algorithms. We show that the utility guarantee depends on the Gaussian width of the unit ball of the dual norm. (iii) We consider the objective perturbation and show that output perturbation on the ERM is equivalent to output perturbation of the dual.

\section{Background}
We consider the empirical risk minimization problem where we have $n$ training data points with each data point represented by a tuple $d_i = (x_i, y_i)$ of the input vector $x_i\in \rb^p$ in a $p$-dimensional space  and the desired output $y_i\in \rb$. We represent the set of all training data points by $D = \{d_1, d_2, \ldots, d_n\}$. We consider {\em generalized linear models} (GLMs) where the goal is to estimate the parameter $\theta$ in a $p$-dimensional space such that $\theta^{\top} x=y$.  Empirical risk minimization considers a loss function the minimizes the error between these two values represented by $l(\theta; d_i)$. For instance, squared loss considers the loss function of the form $l(\theta; d_i)=(\theta^{\top}x_i-y_i)^2$. Note that GLMs may consider different loss functions, such as the squared loss, logistic loss, hinge loss etc. In order to ensure that the model with parameter $\theta$ does not overfit to the (noisy) training data, a regularizer of $\theta$ is used. This regularizer is often represented by $\Omega(\theta)$. Therefore, the empirical risk is 
denoted by $\mathcal{L}(\theta; D)$, where
\BEQ
\mathcal{L}(\theta; D) = \frac{1}{n}\sum_{i=1}^n l(\theta; d_i) + \frac{\lambda}{n}\Omega(\theta),
\label{eq:erm}
\EEQ
where $\lambda$ is the regularizing parameter. The regularizing parameter acts as a Lagrangian multiplier that ensure the presence of the solution in a convex set $\Omega(\theta) \leq c$ for some positive scalar c. The {\em generalized risk} for this problem is given by 
\BEQ
J(\theta; D) = \min_{\theta \in \rb^p} \mathbb{E}_{d_i \sim Dist} [l(\theta; d_i) + \frac{\lambda}{n}\Omega(\theta)].
\label{eq:generalizedRisk}
\EEQ
In this work, we focus on regularizers that induce structured sparsity. There are essentially two forms of sparsity: a) the vector $\theta$ with low non-zero components and b) there are many component in $\theta$ that have the same values. Both these sparse forms may be induced by using submodular functions as the regularizer. We will now briefly review  submodular analysis that will be helpful in understanding this work.

Let us assume a set $V$ representing all dimensions of the vector $\theta$, $V = \{1, 2, \ldots, p\}$. A discrete function $F$ defined on all possible subsets of $V$, $F:2^V \to \rb$ is submodular if 
\begin{align}
\forall A, B \subseteq V: F(A) + F(B) \geq F(A \cup B)  + F(A \cap B).
\end{align}
The power set $2^V$ is naturally identified with the vertices $\{0,1\}^p$ of the hypercube in $p$ dimensions (going from $A \subseteq V$ to $1_A \in \{0,1\}^p$)\footnote{$1_A(i) = 1$ if $i \in A$ and $0$ otherwise.}. Thus, any set-function may be seen as a function $f$ on $\{0,1\}^p$. It turns out that $f$ may be extended to the full hypercube $[0,1]^p$ by piecewise-linear interpolation, and then to the whole vector space $\rb^p$.  This extension is piecewise linear for any set-function~$F$. It turns out that it is convex if and only if $F$ is submodular~\cite{lovasz1982submodular}. Therefore, given a submodular function $F:2^V \to \rb$, we may define its convex extension $f:\rb^p \to \rb$, also referred to as the \lova extension of $F$.

We will now review the use of submodular functions for inducing structured sparsity in $\theta$. Let us assume the support set of $\theta$ as all dimensions with non-zero values defined by $\supp(\theta) = \{j \in V, \theta_j \neq 0\}$. Let us now rewrite the empirical risk in \eq{erm} that may be used to estimate $\theta$ with sparse structures,
%
\BEQ
\mathcal{L}(\theta; D) = \frac{1}{n}\sum_{i=1}^n l(\theta; d_i) + \frac{\lambda}{n}F(\supp(\theta)),
\label{eq:sparseERMD}
\EEQ
where $F$ is a non-decreasing submodular function such that $F(\varnothing) = 0$ and the value of $F$ on all the singletons is strictly positive. We may define a tight convex relaxation of $F(\supp(\theta))$ on the unit $L_{\infty}$ ball $[-1, 1]^p$ as $\Omega_{\infty}: \theta \to f(|\theta|)$, where $f$ is the \lova extension of $F$ and $|\theta|$ is elementwise absolute value of $\theta$. This is known to be a norm~\cite{Bach2010}. The dual norm\footnote{Dual norm of $\Omega$ is $\forall s \in \rb^p,  \Omega^*(s) = \sup_{\Omega(\theta) \leq 1} s^{\top}\theta$} of $\Omega_{\infty}$ is $$\Omega^*_{\infty}(s)~=~\max_{A \subseteq V, A \neq \varnothing}~\frac{|s|(A)}{F(A)}.$$ The unit ball of this dual norm is a well structured polytope called the symmetric submodular polyhedron of $F$, which we define now.

{\em Symmetric submodular polyhedron.} If $F$ is a non-decreasing submodular function such that $F(\varnothing) = 0$  then the symmetric submodular polyhedron is defined as
$$
|P|(F) = \\
\big\{ s \in \rb^p, \ \forall A \subset V, |s|(A) \leqslant F(A) \big\}.
$$
A key result in submodular analysis is that the tight convex relaxation of $F(\supp(\theta))$ is the support function of the symmetric submodular polyhedron of $F$, i.e., 
\BEQ
\Omega_\infty(\theta) = f(|\theta|) = \sup_{s \in |P|(F)} \theta^\top s.
\label{eq:supp}
\EEQ
This may be computed in closed form using a greedy algorithm~\cite{fot_submod}. Please refer to ~\cite{fot_submod,fujishige2005submodular} for more detailed information about submodular functions.

Therefore, the tight convex relaxation of the empirical risk defined in \eq{sparseERMD} is given by
\BEQ
\mathcal{L}(\theta; D) = \frac{1}{n}\sum_{i=1}^n l(\theta; d_i) + \frac{\lambda}{n}\Omega_\infty(\theta).
\label{eq:sparseERM}
\EEQ
If $F$ is a cardinality function, $F(A) = |A|$ then \eq{sparseERMD} is regularized by $L_0$ pseudo-norm, and the tight convex relaxation using submodular analysis gives $\Omega_\infty(\theta)~=~\|\theta\|_1$, which is an $L_1$ norm that induces sparsity. This theory may also be extended to understand group Lasso and other sparsity inducing norms~\cite{fot_submod}. 

\paragraph{Notation.} We introduce some other notions from convex optimization~\cite{rockafellar97} that we use in this work.

{\em Unit ball of the norm $\Omega_{\infty}$}. We use $\mathcal{C}$ to represent the unit ball of the sparsity inducing norm, i.e., 
$$\mathcal{C} = \{\theta \in \rb^p: \Omega_\infty(\theta) \leq 1\}.$$
{\em Gaussian width of a set $\mathcal{C}$.} Let $b \sim \mathcal{N}(0, I_p)$ be a Gaussian random vector in $\rb^p$. The Gaussian width of the set $\mathcal{C}$ is defined as $$G_{\mathcal{C}} \delequal \mathbb{E}_b [ \sup_{w \in \mathcal{C}} b^{\top}w].$$ Note that we use $G_{|P|(F)}$ to represent the Gaussian width of the symmetric submodular polyhedron of $F$.

{\em Gauge function.} Given a closed convex set $\mathcal{C} \subset \rb^p$, then the gauge function $\gamma_{\mathcal{C}}$ is the function defined as $$ \forall \theta \in \rb^p, \gamma_C(\theta) = \inf \{ r \geq 0, \theta \in r \mathcal{C} \}.$$
\paragraph{Problem definition.} We are primarily interested in analyzing the standard differentially private algorithms for empirical risk minimization with sparsity inducing norms, i.e.,
\begin{align}
&\argmin_{\theta \in \rb^p} \frac{1}{n}\sum_{i=1}^n l(\theta; d_i) + \frac{\lambda}{n}\Omega_\infty(\theta)
\label{eq:ERM1}\\
&\argmin_{\theta \in \mathcal{C}} \frac{1}{n}\sum_{i=1}^n l(\theta; d_i).
\label{eq:ERM2}
\end{align}
Note that \eq{ERM1} and \eq{ERM2} are equivalent problems. This is an {\em improper learning} setting where the true risk may lie in a bounded set $\mathcal{C}$, the optimal of the the empirical risk minimization in \eq{ERM1} is allowed to produce a hypothesis from $\rb^p$.

In this work, we use $\theta^*$ as the optimal solution, i.e., 
\BEQ
\theta^* = \argmin_{\theta \in \rb^p} \mathbb{E}_{d_i \sim Dist} [l(\theta; d_i) + \frac{\lambda}{n}\Omega(\theta)],
\label{eq:opt}
\EEQ
and the optimal empirical solution $\hat{\theta}$ is given by
\BEQ
\hat{ \theta} = \frac{1}{n} \sum_{i=1}^n l(\theta; d_i) + \frac{\lambda}{n}\Omega(\theta).
\label{eq:eopt}
\EEQ

{Dual derivation.} In this paper, we extensively use the dual optimization problem to the ERM in \eq{ERM1}. For brevity, let us assume $\hat{\mathcal{L}}(\theta) = \sum_{i=1}^n l(\theta;d_i)$. Then the primal optimization problem is given by
\BEQ
\min_{\theta \in \rb^p} \hat{\mathcal{L}}(\theta) + \frac{\lambda}{n}\Omega_\infty(\theta)
.%
\label{eq:ERM1Primal}
\EEQ
Let us denote its optimal solutions by $\hat{\theta}$. The dual to the empirical risk minimization is 
\BEQ
\sup_{s \in K} - \hat{\mathcal{L}}^*(-s),
\label{eq:ERM1Dual}
\EEQ
where $\hat{\mathcal{L}}^*$ is the Fenchel conjugate of $\hat{\mathcal{L}}$ and $K$ is a polytope given by $\frac{\lambda}{n}|P|(F)$. The primal and dual solutions $\theta \in \rb^P$ and $s \in |P|(F)$ are related by
\begin{align}
\theta = \argmin_{\theta \in \rb^p} \hat{\mathcal{L}}(\theta) +  \frac{\lambda}{n}{}s^{\top} \theta.
\end{align}

The derivation of the dual is provided in the supplementary material

\begin{table*}[htb]
\begin{center}
\begin{tabular}{ccccccc}
\hline
Optimization & Privacy        & Loss & Regularizer & Risk             & Excess Risk & Reference\\
\hline
 SGD 	& $\epsilon$           & -   & -           & Convex          & $O(\frac{p}{\epsilon \sqrt{n}})$& \cite{Wu2017} \\
 SGD 	& $\epsilon$           & -   & -           & Strongly Convex & $O(\frac{p}{\epsilon n})$       & \cite{Wu2017} \\
 Exact  & $(\epsilon, \delta)$ & GLM & $L_2$       &  -              & $O(\frac{1}{\epsilon\sqrt{n}})$  & \cite{jain14} \\
 SGD    & $(\epsilon, \delta)$ & Strongly Convex   & Smooth          & Strongly Convex              & $O(\frac{p}{\epsilon^2n^2})$  & \cite{zhang2017} \\
 SGD    & $(\epsilon, \delta)$ & Strongly Convex   & Smooth          & Convex                       & $O\bigg(\big(\frac{\sqrt{p}}{\epsilon n}\big)^{\frac{2}{3}}\bigg)$  & \cite{zhang2017} \\
\hline
\end{tabular}
\end{center}
\caption{Output Perturbation}
\label{table:OutP}
\end{table*}

\begin{table*}[htb]
\begin{center}
\begin{tabular}{ccccccc}
\hline
Optimization & Privacy & Loss & Regularizer  & Risk & Excess Risk & Reference\\
\hline
 Exact  & $(\epsilon, \delta)$ &  - & $L_2$ &  Strongly Convex  &$ O(\frac{p}{\epsilon^2n^2})$  & \cite{chaudhuri2011} \\
 Exact  & $(\epsilon, \delta)$ &  - & Non-Smooth &  Strongly Convex  &$ {\tilde O}(\frac{p}{\epsilon^2n^2})$  & \cite{Kifer12} \\
 Exact  & $(\epsilon, \delta)$ &  - & Non-Smooth &  Convex  &$ {\tilde O}(\frac{\sqrt{p}}{\epsilon n})$  & \cite{Kifer12} \\
 Exact  & $(\epsilon, \delta)$ & GLM & $L_2$ &  -  &${\tilde O}(\frac{1}{\epsilon\sqrt{n}})$  & \cite{jain14} \\

\hline
\end{tabular}
\end{center}
\caption{Objective Perturbation}
\label{table:ObjP}
\end{table*}

\begin{table*}[htb!]
\begin{center}
\begin{tabular}{ccccccc}
\hline
Optimization & Privacy & Loss & Regularizer & Risk & Excess Risk & Reference\\
\hline
SGD 		   & $(\epsilon, \delta)$ & - & - & Convex & ${\tilde O}(\frac{\sqrt{p}}{\epsilon n})$& \cite{Bassily2014}\\
SGD 		   & $(\epsilon, \delta)$ & - & - & Strongly Convex & ${\tilde O}(\frac{\sqrt{p}}{\epsilon^2 n^2})$& \cite{Bassily2014} \\
 Frank-Wolfe & $(\epsilon, \delta)$ & Squared & $L_1$ &  LASSO  &${\tilde O}\bigg(\frac{1}{n^{\frac{2}{3}}}\bigg)$  & \cite{Talwar2015} \\
\hline
\end{tabular}
\end{center}
\caption{Gradient Perturbation}
\label{table:GraP}
\end{table*}

\section{Related work}
Let us assume $\epsilon \geq 0$, $\delta \in [0,1]$. Let $\mathcal{X}$ be the domain of $x$ and let us consider two datasets $\mathcal{D}, \mathcal{D'}$ of $n$ points that lie in $\mathcal{X}^n$. Let $d_H$ denote the Hamming distance between the two datasets. We have the following definitions of differential privacy and approximate differential privacy.
\begin{definition}{\bf Differential Privacy~\cite{Dwork2014}.}
	A randomized mechanism $h:\mathcal{X}^n \to \mathcal{Y}$ is $\epsilon$-differentially private if for any $d_H(\mathcal{D},\mathcal{D'}) = 1$ and for all $E \subseteq \mathcal{Y}$, we have
$$
\mathbb{P}[h(x) \in E] \leq e^{\epsilon}\mathbb{P}[h(x') \in E].
$$
\end{definition}

\begin{definition}{\bf Approximate Differential Privacy~\cite{Dwork2014}} A randomized mechanism $h:\mathcal{X}^n \to \mathcal{Y}$ is $(\epsilon, \delta)$-differentially private if for any $d_H(\mathcal{D},\mathcal{D'}) = 1$ and for all $E \subseteq \mathcal{Y}$, we have
\begin{align}
\mathbb{P}[h(x) \in E] \leq e^{\epsilon}\mathbb{P}[h(x') \in E] + \delta.
\end{align}
\end{definition}
The privacy parameters $\epsilon$ and $\delta$ need to be small to ensure stronger privacy.  \cite{Kasi2008} suggest that $\epsilon$ is a small constant and $\delta$ should be $o(1/n^2)$. They also showed that $\delta$ does not depend on $p$ to  achieve a fixed level of privacy. 

Different algorithms use {\em global sensitivity framework} to design differentially private algorithms. Global sensitivity of a mechanism $h$ is defined as
$$GS(h) = \max_{\mathcal{D}, \mathcal{D'} \in \mathcal{X}^n, d_H(\mathcal{D}, \mathcal{D'}) = 1} \| h(\mathcal{D}) - h(\mathcal{D'})\|_q, $$
where $\|.\|_q$ is the $L_q$ norm. \cite{Dwork2006} show that, given any dataset $\mathcal{D}$, adding a noise vector $b \in \rb^p$ to the output $h(\mathcal{D})$ that is sampled from a distribution with density proportional to $e^{-\frac{\epsilon\|b\|_q}{GS(h)}}$ is $\epsilon$-differentially private. This is referred to as {\em output perturbation.} The mechanism $h$ in the context of empirical risk minimization is to minimize the empirical error defined
in \eq{sparseERM}. {\em Objective perturbation} is a randomized mechanism to add a random linear term to this objective function and obtain a differentially private output. {\em Gradient perturbation} is another randomized mechanism to add noise to the gradients while optimizing the empirical error in \eq{sparseERM}. 

Let $\theta^{priv}$ be the solution of a differentially private algorithm, $\mathcal{A}$, which is a randomized mechanism. The excess risk in the context of empirical risk minimization is given by a bound on
\begin{align} 
\mathbb{E}_{\mathcal{A}}[J(\theta^{priv};D) - J(\theta^*;D)],
\label{eq:trueRisk}
\end{align}
where $ J(\theta; D) = \min_{\theta \in \rb^p} \mathbb{E}_{d_i \sim Dist} [l(\theta; d_i) + \frac{\lambda}{n}\Omega(\theta)]$ that gives the generalization error. Note that expectation in \eq{trueRisk} is over the randomness of the algorithm. For brevity, we drop the subscript $\mathcal{A}$ in the rest of the paper. All the randomized mechanisms described earlier have been extensively studied in literature in the context of empirical risk minimization. Table~\ref{table:OutP} refers to the excess risk due to output perturbation. Table~\ref{table:ObjP} gives the excess risk due to objective perturbation.  Table~\ref{table:GraP} gives the excess risk due to gradient perturbation.
\cite{Abadi2016} provided a deep learning based differentially private algorithm.

In the next section, we provide these excess risk bound for empirical risk minimization problems that use sparsity inducing norms as regularizers. Our main contribution is the use of submodular theory to provide excess risk for output perturbation that is dependent on the Gaussian width of the symmetric submodular polyhedra. We also consider the use the differentially private Frank-Wolfe algorithm to optimize the dual of the empirical risk minimization and derive excess risk. Further, we show that objective perturbation of the ERM is equivalent to output perturbation of the dual.

\section{Sparse private risk minimization}
In this section, we consider the empirical risk minimization problem in \eq{sparse} to provide differentially private algorithms and the corresponding utility guarantees for output perturbation algorithm.
\BEQ
\mathcal{L}(\theta;D) = \frac{1}{n} \sum_{i=1}^n l(\theta; d_i) + \frac{\lambda}{n} \Omega_{\infty}(\theta).
\label{eq:sparse}
\EEQ
Let $\theta^{priv}$ be the differentially private solution given by a differentially-private algorithm $\mathcal{A}$. Then the utility guarantee for the algorithm $\mathcal{A}$ is given by
$$
\mathbb{E}[J(\theta^{priv};D) - J(\theta^*;D)],
$$
where $J(\theta; D) = \min_{\theta \in \rb^p} \mathbb{E}_{d_i \sim Dist} [l(\theta; d_i) + \frac{\lambda}{n}\Omega(\theta)]$. We use the following result from~\cite{Shalev09}.

\begin{proposition}[Theorem 1 of~\cite{Shalev09}]
\label{prop:general}
Let $B$ be the bound on the domain of $\theta$, $R_2$ is the $L_2$ bound on the norm of any input vector $x$, $l$ is $L$-Lipschitz in $\theta$. Then for any distribution on $d$ and any $\alpha > 0$, with probability at least $1 - \alpha$,
\BEAS
J(\theta^{priv};D) - J(\theta^*;D)& \leq &\mathcal{L}(\theta^{priv};D) - \mathcal{L}(\hat{ \theta};D)+ O\bigg({\sqrt \frac{(BR_2L)^2log(1/\alpha)}{n}}\bigg),
\EEAS
where $\theta^* = \argmin_{\theta \in \rb^p} J(\theta;D)$ and $\hat{ \theta} = \argmin_{\theta \in \rb^p} \mathcal{L}(\theta;D)$.
\end{proposition}

\subsection{Output perturbation}
The output perturbation first computes the optimal empirical risk minimizer $\hat{ \theta}$ for $\mathcal{L}(\theta; D)$ and adds noise according to the sensitivity of $\hat{ \theta}$. The output perturbation algorithm is given by
\BEQ
\theta^{priv} = \hat{ \theta} + b,
\label{outputPert}
\EEQ
where $b \in \rb^p$ is a random vector. \cite{chaudhuri2011} showed that if $b$ is generated using a Gamma distribution with kernel $e^{-\frac{v\epsilon\lambda}{4LR_2}}$ and the double derivative of the loss function $l(\theta;d_i)$ is $c_2$, then the excess risk is bounded by 
$$
O\bigg(\frac{Lc_2^{1/3}p\log p(R_2 \|\theta^*\|_2)^{4/3}}{\epsilon \sqrt{ n}}\bigg).
$$ 
This was proven to be $\epsilon$-differentially private. It was further modified by~\cite{Kifer12} to generate $b$ using a Gaussian distribution $\mathcal{N}(0, \sigma^2 I_p)$, where $\sigma^2~=~\frac{16 (LR_2)^2(log(1/\delta) + \epsilon)}{\lambda^2 \epsilon^2}$. In this case, the excess error improved to
$$
O\bigg( \frac{L c_2^{1/3}(R_2\|\theta^*\|_2)^{4/3}\sqrt{p(\log(1/\delta + \epsilon)}}{\epsilon \sqrt{ n}}\bigg).
$$
However, this was proven to be $(\epsilon, \delta)$ approximate differentially private. This analysis for the same algorithm was further improved by~\cite{jain14} to give dimensionality independent excess risk
$$
O\bigg(\frac{LR_2\|\theta^*\|_2\sqrt{(\log(1/\delta + \epsilon)}}{\epsilon \sqrt{ n}}\bigg).
$$ 
It may be noted that all these analyses make an assumption that the regularization function is the squared $L_2$ norm. In this work, we primarily use the analysis by~\cite{jain14} to extend this to regularizers that induce structured sparsity given by $\Omega_{\infty}(\theta)$. Note that these are convex, non-smooth, non-differentiable functions. 

\begin{theorem}[Utility guarantee]
Let $D = \{d_1, d_2, \ldots, d_n\}$ be i.i.d. samples drawn from a fixed distribution $Dist$ from $\rb^p$. Let $F$ be a non-decreasing submodular function such that $F(\varnothing) = 0$ with all the singletons strictly positive, $|P|(F)$ denotes its symmetric polymatroid and the $G_{|P|(F)}$ is the Gaussian width of $|P|(F)$. Let $\Omega_{\infty}(\theta)$ be the tight convex relaxation of $F(\supp(\theta))$ that is used as the regularizer and $\lambda = \frac{LR_2\sqrt{n}}{G_{|P|(F)}}$, then with probability at least $1 - \alpha$ over the randomness of the training data set $D$ it holds that
\begin{align}
\begin{aligned}
    \mathbb{E}&[J(\theta^{priv};D) - J(\theta^*;D)] \\
    &= O\bigg(\frac{L R_2 G_{|P|(F)}\sqrt{\log(1/\delta) + \epsilon}}{\epsilon{\sqrt n}}\bigg).
    \end{aligned}
\end{align}
\label{thm:outputP}
\end{theorem}
In order to prove this theorem we need to prove the following lemma.

\begin{lemma} Using the assumptions in Theorem~\ref{thm:outputP}
$$ 
\mathbb{E}[\mathcal{L}(\theta^{priv};D)  - \mathcal{L}(\theta^*;D)] = O(LR_2 \sigma + \frac{\lambda}{n} G_{|P|(F)}).
$$
\label{lemma:outputP}
\end{lemma}
\begin{proof}
Using the Lipschitz property of $l$, we have
\BEAS
\mathcal{L}(\theta^{priv};D) - \mathcal{L}(\theta^*;D) 
&  \leq  & \frac{1}{n}\sum_{i=1}^n L|(\theta^{priv} - \theta^*)^\top x_i| +  \frac{\lambda}{n} (\Omega_{\infty}(\theta^* + b) -  \Omega_{\infty}(\theta^*)) \\
&  =     &  \frac{L}{n}\sum_{i=1}^n |b^\top x_i|  + \frac{\lambda}{n}(f(|\theta^* + b|) -  f(|\theta^*|)) \\
&  \leq  &  \frac{L}{n}\sum_{i=1}^n |b^\top x_i|  + \frac{\lambda}{n}(f(|\theta^*| + |b|) -  f(|\theta^*|)) \\
&       & \text{(using monotonicity of $f$)} \\
&  \leq  &  \frac{L}{n}\sum_{i=1}^n |b^\top x_i|  + \frac{\lambda}{n} f(|b|) \\
&        & \text{(using submodularity of $f$)}
\EEAS
The last two inequalities come from the following properties of $F$, $\forall A \subseteq B \subseteq V$ and $\forall b \in V \setminus B$, we have
\BEA
F(A) & \leq & F(B) \label{eq:mono} \\
F({b} \cup A) - F(A) & \geq & F(\{b\} \cup B) - F(B). \label{eq:dimini} 
\EEA
The non-decreasing assumption of the submodular function leads to \eq{mono}. The diminishing returns property of the submodular function is given by \eq{dimini}. Therefore, these properties also extend to the tight convex extension $f$ of $F$.
Hence,
\BEAS
\mathbb{E}_b(\mathcal{L}(\theta^{priv};D) - \mathcal{L}(\theta^*;D))  &  \leq  &  \mathbb{E}_b(\frac{L}{n}\sum_{i=1}^n |b^\top x_i|  + \frac{\lambda}{n} f(|b|)) \\
&  \leq  &  \frac{L\sigma}{n}\sum_{i=1}^n \|x_i\|_2  + \frac{\lambda}{n} \mathbb{E}_b(f(|b|)) \\
&  \leq  &  L R_2 \sigma + \frac{\lambda}{n} \mathbb{E}_b(f(|b|)) \\
&  =     &  L R_2 \sigma + \frac{\lambda}{n} \mathbb{E}_b(\sup_{s \in |P|(F)} s^{\top}b)\\
&  = &  L R_2 \sigma + \frac{\lambda}{n} G_{|P|(F)}
\EEAS

That last two equalities use the property of the regularizer that from \eq{supp} that  $\Omega_{\infty}(\theta) = f(|\theta|) = \sup_{s \in |P|(F)} s^{\top}\theta$ and the definition of the Gaussian width of a set.
\end{proof}

\begin{proof}{\bf of Theorem~\ref{thm:outputP}}
Using Proposition~\ref{prop:general}, we have
\BEAS
\mathbb{E}[J(\theta^{priv};D) - J(\theta^*;D)] & \leq & \mathbb{E}[\mathcal{L}(\theta^{priv};D) - \mathcal{L}(\hat{ \theta};D)]  + O\bigg({\sqrt \frac{(BR_2L)^2log(1/\alpha)}{n}}\bigg)\\
& \leq & O(L R_2 \sigma) + \frac{\lambda}{n} G_{|P|(F)}  + O\bigg({\sqrt \frac{(BR_2L)^2log(1/\alpha)}{n}}\bigg) \\
&       & \text{(using Lemma~\ref{lemma:outputP})} \\
&  \leq & O\bigg(\frac{L R_2 G_{|P|(F)}\sqrt{\log(1/\delta) + \epsilon}}{\epsilon{\sqrt n}}\bigg).
\EEAS
The last inequality follows from  $$\sigma = \frac{LR_2\sqrt{\log(1/\delta) + \epsilon}}{\lambda \epsilon}$$ and $$\lambda = \frac{LR_2\sqrt n}{G_{|P|(F)}},$$ which concludes the proof.
\end{proof}

This result of Theorem~\ref{thm:outputP} is similar to the results of~\cite{jain14} but it also holds for non-smooth regularizers. Excess risk in this case is only dependent on the Gaussian width of the symmetric submodular polyhedron which is a the unit ball of the dual norm.

If  the cardinality function is $F(A) = |A|$, then the tight convex relaxation, $\Omega_{\infty}$ is the $L_1$ norm while the symmetric submodular polyhedron is the dual norm, which is the $L_{\infty}$ norm. The Gaussian width of an $L_{\infty}$ ball is $O(p)$. If we use $F(A) = \min\{|A|,1\}$ then the tight convex relaxation is the $L_{\infty}$ norm and its dual norm is the $L_{1}$ norm. The Gaussian width of the $L_1$ norm is $O(\sqrt{\log p})$. However, we may not be able to estimate the Gaussian width of all possible norms.

\begin{algorithm*}[htb!]
\caption{Frank-Wolfe \cite{fot_submod}}
\label{Alg:FWAlg}
\KwData{Dual cost function $\hat{\mathcal{L}}^*:K \to \rb$. This is computed using all the data $D = (d_1, \ldots, d_n)$ and the empirical loss given by $\mathcal{L}(\theta, D)$}
\KwResult{ $\hat{s}$} 
Choose  an arbitrary $s_0 \in K$, \\
\For{$t = 1$; $i \leq T$; $t = t+1$ }
{
$\rho_t = \frac{1}{t+2}$ \\
$\bar{s}_t \gets \argmin_{s \in K} (s_{t-1} - s)^{\top} \nabla \hat{\mathcal{L}}^*( - s_{t-1})$ \\
$s_t \gets (1 - \rho_t)s_{t-1} + \rho_t\bar{s}_t $
}
$\hat{s} \gets s_T$
\end{algorithm*}

\begin{algorithm*}[htb!]
\caption{Private Frank-Wolfe}
\label{Alg:PrivFWAlg}
\KwData{ Dual cost function $\hat{\mathcal{L}}^*:K \to \rb$ (with $L_1$-Lipschitz constant $L$), convex set $K$, set of its extreme points $S$ and Gaussian Width $G_K$, privacy parameters: $(\epsilon, \delta)$}
\KwResult{ $\hat{s}^{priv}$} 
Choose  an arbitrary $s_0 \in K$, \\
\For{$t = 1$; $i \leq T$; $t = t+1$ }
{
$\forall s \in S, \alpha_s \gets  (s_{t-1} - s)^{\top} \nabla \hat{\mathcal{L}}^*( - s_{t-1}) + \text{Lap}\bigg(\frac{L\Gamma_K\sqrt{8T\log{1/\delta}}}{n\epsilon}\bigg),$ where Lap$(\lambda) \sim \frac{1}{2\lambda}\exp\big(-\frac{|x|}{\lambda}\big)$\\
$\rho_t = \frac{1}{t+2}$ \\
$\bar{s}_t \gets \argmin_{s \in S} \alpha_s $\\
$s_t \gets (1 - \rho_t)s_{t-1} + \rho_t\bar{s}_t $
}
$\hat{s}^{priv} \gets s_T$
\end{algorithm*}
\subsection{Private Frank-Wolfe algorithm}
In this section, we consider the Frank-Wolfe algorithm~\cite{frank1956algorithm} for empirical risk minimization. It is a classical greedy algorithm that relies on first order approximation of the function and moves towards the optima. Further, we propose a private version of the Frank-Wolfe for providing a differentially private solution. We also analyze the utility guarantee of the algorithm.

Frank-Wolfe algorithm for empirical risk minimization gathered attention with the availability of efficient algorithms for performing linear optimization on well structured polytopes. For instance, LASSO where the domain is an $L_1$ norm ball. However, \cite{fot_submod} also showed that the dual of the ERM with sparsity inducing norms may be optimized using the Frank-Wolfe algorithm. We recall that the dual of the ERM is given by
\BEQ
\sup_{s \in K} - \hat{\mathcal{L}}^*(-s),
\label{eq:ERM1Dual}
\EEQ
where $\hat{\mathcal{L}}^*$ is the Fenchel conjugate of $\hat{\mathcal{L}}$ and $K$ is a centrally symmetric polytope. The algorithm that optimizes the dual is given in Algorithm~\ref{Alg:FWAlg}. It is an ideal algorithm to optimize the dual given in \eq{ERM1Dual} because linear functions may be optimized on the polytope $K$ using a greedy algorithm\cite{fujishige2005submodular, fot_submod}. 

We assume the the loss function $\hat{\mathcal{L}}$ is $\Delta$-strongly convex. The dual cost function in \eq{ERM1Dual} $\hat{\mathcal{L}}^*$ is $\frac{1}{\Delta}$-strongly smooth~\cite{Kakade2009OnTD}. We now have the following convergence properties for the Frank-Wolfe algorithm.
\begin{theorem}[\cite{fot_submod, jaggi13}]
Algorithm~\ref{Alg:FWAlg} has the following convergence rate
$$\hat{\mathcal{L}}^*(s_T) - \inf_{s \in K} \hat{\mathcal{L}}^*(s) \leq O\bigg(\frac{\Gamma_K^2}{\Delta T}\bigg).$$
\end{theorem}

Private versions of Frank-Wolfe have also been proposed for several machine learning tasks. In this work, we restrict to empirical risk minimization. \cite{talwar2014, Talwar2015} provided a private version of the Frank-Wolfe algorithm to optimize the standard empirical risk minimization problems, for instance LASSO. They provided near optimal bounds for LASSO. In their setting, the assumption is that domain of the primal optimization problem, i.e., $\Omega_\infty(\theta) \leq 0$ is a polytope with polynomial constraints. This is specifically true in the case of sparse regularizers. However, it might not be very intuitive or easy to approximate linear functions on these polytopes. Also, there are no guarantees of better convergence than gradient descent algorithms. Therefore, we propose a private version of the Frank-Wolfe to optimize the dual of the ERM in Algorithm~\ref{Alg:PrivFWAlg}. Note that the algorithm is motivated from private Frank-Wolfe polytope version of \cite{talwar2014}.

Let $\hat{\mathcal{L}}^*$ be $L$-Lipschitz with respect to the $L_1$-norm and let $\Gamma_K$ be the diameter of the set $K$. Here, we perturb the linear optimization in the inner loop with a random vector $b$ drawn randomly from a Laplacian distribution.

\begin{theorem}[Privacy guarantee] Algorithm~\ref{Alg:PrivFWAlg} is $(\epsilon, \delta)$-differentially private.
\end{theorem}
It has shown to be $(\epsilon, \delta)$-differntially private by \cite{talwar2014}. We use the same utility gurantees shown by~\cite{talwar2014} but for the dual optimization problem.

\begin{theorem}[Utitlity Guarantee] Let $L$ be the $L_1$-Lipschitz constant of $\hat{\mathcal{L}}$ defined on a convex set $K$ with Gaussian width $G_K$. Let $\Gamma_K$ be the diameter of the convex set $K$. In Algorithm~\ref{Alg:PrivFWAlg} if we set $T = \frac{\Gamma_K^{4/3}(n\epsilon)^{2/3}}{(LG_K)^{2/3}}$, then
\BEAS
\mathbb{\hat{\mathcal{L}}^*}(\hat{s}^{priv})  -  \inf_{s \in K} \hat{\mathcal{L}}^*(s) =  O\bigg(\frac{\Gamma_K^{2/3}(L G_{K})^{2/3} \log(n|S|) \sqrt{\log(1/\delta)}}{(n\epsilon)^{2/3}}\bigg).
\EEAS
\end{theorem}
The theorem is exactly the application of the Theorem-5.5 proved in \cite{talwar2014}. However, this is applied to the dual polytopes that are well structured in the sparsity inducing norms setting\cite{fot_submod}. The Frank-Wolfe algorithm may also be used to solve empirical risk minimization problems with combinatorial penalties~\cite{obozinski2012convex}. 

In this section, we have used the polytope version of the Frank-Wolfe algorithm to optimize the dual of the empirical risk minimization problem. The excess risk depends on the Gaussian width of the polytope $K$. However, it is worth noting that the polytope $K$ is a scaled version of the $|P|(F)$. An appropriate choice of $\lambda = O(n)$ will make the excess risk dependent on the Gaussian width of symmetric submodular polyhedron, $|P|(F)$.

\subsection{Objective perturbation}
In this section, we consider objective perturbation, i.e., perturbation of the objective function $\mathcal{L}(\theta;D)$ by a random linear vector. The output perturbation that was initially considered by~\cite{chaudhuri2011} was of the form
\BEQ
\theta^{priv} = \argmin_{\theta \in \rb^p} \sum_{i=1}^n l(\theta;d_i) + \frac{\lambda \|\theta\|_2^2}{2n} + \frac{b^{\top}{\theta}}{n}.
\EEQ
They initially used a Gamma distribution to generate $b$. They also proved that that it was $\epsilon$-differential private with an excess risk bound
$$
O\bigg( \frac{LR_2\|\theta^*\|_2 p \log p}{\epsilon \sqrt{n}}\bigg).
$$
This was further improved by \cite{Kifer12} by using a Gaussian noise vector $b$ that was proven to satisfy $(\epsilon, \delta)$ approximate differential privacy. The excess risk bound was improved to 
$$
O\bigg( \frac{LR_2\|\theta^*\|_2 \sqrt{p \log(1/\delta)}}{\epsilon \sqrt{n}}\bigg).
$$
\cite{jain14} used the same algorithm to provide a dimensionality-independent excess risk bound
$$
O\bigg( \frac{(\log n)^2(LR_2)^2\|\theta^*\|_2 \sqrt{\log(1/\delta) + \epsilon}}{\epsilon \sqrt{n}}\bigg).
$$
Empirical risk minimization problems with non-smooth regularizers were considered by \cite{talwar2014, Talwar2015}. Often the unit balls of the norms used as regularizers form a constraint set $\mathcal{C}$. It may be defined as
$$
\mathcal{C} = \{ \theta \in \rb^p: \Omega_{\infty}(\theta) \leq 1 \}.
$$
The minimization of empirical risk minimization in \eq{sparse} is now
\BEQ
\min_{\theta \in \mathcal{C}} \frac{1}{n} \sum_{i=1}^n l(\theta; d_i).
\label{eq:sparseERMOP}
\EEQ
This is equivalent to 
$$ 
\min_{\theta \in \rb^p} \frac{1}{n} \sum_{i=1}^n l(\theta; d_i) + \frac{\lambda}{n} \Omega_{\infty}(\theta).
$$
We represent the optimal solution of \eq{sparseERMOP} be $\hat{ \theta}$. Assume the loss function is the squared error, the domain of the data points is given by $$\{(x,y): x \in \rb^p, y \in [-1, 1], \| x \|_{\infty} \leq 1 \}$$ and the convex set $\mathcal{C}$ is a unit $L_1$ ball. Then this setting is the standard LASSO setting. The private algorithm with objective perturbation for sparse structures is proposed by~\cite{talwar2014, Kifer12}. 
Objective perturbation is proven to be $(\epsilon, \delta)$-differentially private with the following utility guarantees that use the Gaussian width of $\mathcal{C}$: Let $\mathcal{C}$ have diameter $\Gamma_{\mathcal{C}}$ and Gaussian width $G_{\mathcal{C}}$. For all data points $d \in D$, let the loss function $l(\theta; d)$ be twice continuously differentiable for all $\theta \in \mathcal{C}$ and its Hessian has a spectral norm of at most $\lambda_{max}$ then proposed algorithm has the following guarantees.
\begin{list} {\labelitemi}{\leftmargin=1.1em}
\addtolength{\itemsep}{-.3\baselineskip}

\item {\bf Lipschitz case.} If for all data points $d \in D$, the loss function is convex and $L$-Lipschitz w.r.t. the $L_2$ norm then the excess risk bound is 
$$ O\bigg(\frac{LG_{\mathcal{C}}\sqrt{1/\delta} + \lambda_{max} \Gamma_{\mathcal{C}}}{n \epsilon}\bigg).$$
\item {\bf Lipschitz and strongly convex case.} If for all data points $d \in D$, the loss function is convex and $L$-Lipschitz w.r.t. the $L_2$ norm and $\Delta$-strongly convex with respect to the gauge function $\gamma_{\mathcal{C}}$ and $\Delta \geq \frac{2 \Gamma_{\mathcal{C}} \lambda_{max}}{n \epsilon}$, then the excess risk bound is
$$O\bigg(\frac{(LG_{\mathcal{C}})^2\sqrt{1/\delta} + \lambda_{max} \Gamma_{\mathcal{C}}}{\Delta (n \epsilon)^2}\bigg).$$
\end{list}

All the above approaches assume a structure of the closed set $\mathcal{C}$ and exploit that to understand privacy properties of objective perturbation.

We will now derive a dual to the empirical risk minimization problem. Further, we show that the objective perturbation of the primal is equivalent to dual. We now have the following proposition.

\begin{proposition}
Objective perturbation of the primal empirical risk minimization in \eq{ERM1Primal} is equivalent to output perturbation of the dual in \eq{ERM1Dual}. 
\end{proposition}
\begin{proof}

Let us now consider the objective perturbed function where $b \sim \mathcal{N}(0, \sigma^2 I_p)$ and $\sigma^2 = \frac{L^2 2 \log(1/\delta)}{(n \epsilon)^2}$. We obtain the objective perturbed risk as 
$$
\theta^{priv} = \argmin_{\theta \in \rb^p} \hat{\mathcal{L}}(\theta) + \frac{\lambda}{n}\Omega_\infty(\theta) + \frac{b^{\top}\theta}{n}.
$$
The dual optimization of this objective perturbed risk is given by
\BEA
&   & \min_{\theta \in \rb^p} \hat{\mathcal{L}}(\theta) + \frac{\lambda}{n} \Omega_{\infty}(\theta) + \frac{b^{\top}\theta}{n} \nonumber \\
& = & \min_{\theta \in \rb^p} \hat{\mathcal{L}}(\theta) + \frac{\lambda}{n} \sup_{s \in |P|(F)} s^{\top} \theta + \frac{b^{\top}\theta}{n} \nonumber \\
&   & \text{(as $\Omega_{\infty}$ is the support function of $|P|(F)$)} \nonumber \\
& = & \min_{\theta \in \rb^p} \sup_{s \in |P|(F)} \hat{\mathcal{L}}(\theta) +  \frac{\lambda}{n} s^{\top} \theta  + \frac{b^{\top}\theta}{n} \nonumber \\ 
&   & \text{(due to strong duality)} \nonumber \\
& = & \sup_{s \in |P|(F)} \min_{\theta \in \rb^p} \hat{\mathcal{L}}(\theta) +  \frac{\lambda}{n} s^{\top} \theta + \frac{b^{\top}\theta}{n} \nonumber \\
& = & \sup_{s \in |P|(F)} - \hat{\mathcal{L}}^*\bigg(-\frac{\lambda s + b}{n}\bigg)   \nonumber\\
&   & \text{(where $\hat{\mathcal{L}}^*$ is the Fenchel conjugate of $\hat{\mathcal{L}}$)} \nonumber \\
& = & \sup_{s \in K} - \hat{\mathcal{L}}^*\bigg(- \bigg(s + \frac{b}{n}\bigg)\bigg) \label{eq:ERM2Dual}
\EEA
\end{proof}
Let $s^{priv}$ be the optimal solution of the dual in \eq{ERM2Dual}. This may be achieved by perturbing the output of the dual optimization problem $\hat{s} = \argmax_{s \in K} -\hat{\mathcal{L}}^*(-s)$ with the vector $\frac{b}{n}$. The private primal solution is given by
$$
\theta^{priv} = \argmin_{\theta \in \rb^p} \hat{\mathcal{L}}(\theta) +  \frac{\lambda}{n} \theta^{\top} s^{priv} + \frac{b^{\top}\theta}{n}.
$$
However, we have not been able to derive excess risk for objective perturbation. If $\hat{\mathcal{L}}$ is a strongly convex function then the corresponding dual $\hat{\mathcal{L}}^*$ is a strongly smooth function. This information along with Lipschitz assumptions may be used to derive excess risk bounds for objective perturbation.

\section{Conclusion and future work} 
In this work, we mainly considered the standard differential private algorithms, such as output perturbation and objective perturbation, to solve empirical risk minimization problems with non-smooth regularizers that induce structured sparsity. We are the first to provide bounds on output perturbation that depend on the Gaussian width of the dual norm. 

We showed that the existing tight bounds for objective perturbation hold for this class of empirical risk minimization problem. We also derived the respective duals to leave an open problem of how the structure of the unit ball of the dual norm may be exploited to achieve better bounds.

Differentially private algorithms that are based on Frank-Wolfe and mirror-descent have already been analyzed by~\cite{talwar2014} for similar settings. Similar to objective perturbation their results hold for these class of empirical risk minimization problems. For LASSO, they have proved that the excess risk is bounded by ${\tilde O}\big(1/{n^{\frac{2}{3}}}\big)$, which is a very strong result. We have showed that the private version of Frank-Wolfe algorithm to optimize the dual of the empirical risk minimization. We have shown excess risk bounds that are dependent on the Gaussian width of the polytope similar to output perturbation. However, the Frank-Wolfe algorithm is equivalent to mirror-descent on the primal. Note that this is different from the mirror-descent algorithm proposed by~\cite{talwar2014}. As future work, we would like to analyze these class of mirror-descent differentially private algorithms.

Gradient perturbation based algorithms~\cite{Wang2017} have used differentially private algorithms to provide optimal to near optimal risk bounds. They have shown this for both smooth and non-smooth regularizers. We will explore the possibility of better results for the empirical risk minimization with sparsity inducing norms. Importantly, \cite{zhang2017} has given differentially private algorithms for smooth objectives. We would like to understand the performance of these class of differentially private algorithms for ERM with non-smooth regularizers. We would also like to propose the private version of combinatorial algorithms~\cite{fot_submod} to solve class of problems.

\paragraph{Acknowledgments}
This research was partially funded by the Leverhulme Centre for the Future of Intelligence and the Data Science Institute at Imperial College London.

\bibliography{privacy}

\begin{thebibliography}{10}

\bibitem{Abadi2016}
M.~Abadi, A.~Chu, I.~Goodfellow, H.~B. McMahan, I.~Mironov, K.~Talwar, and
  L.~Zhang.
\newblock {Deep Learning with Differential Privacy}.
\newblock In {\em Proc. ACM SIGSAC CCS}, 2016.

\bibitem{Bach2010}
F.~Bach.
\newblock {Structured Sparsity-inducing Norms Through Submodular Functions}.
\newblock In {\em Proc. NIPS}, 2010.

\bibitem{fot_submod}
F.~Bach.
\newblock {\em {Learning with Submodular Functions: A Convex Optimization
  Perspective}}.
\newblock Foundations and Trends in Machine Learning. Now, 2013.

\bibitem{Bassily2014}
R.~Bassily, A.~Smith, and A.~Thakurta.
\newblock {Private Empirical Risk Minimization: Efficient Algorithms and Tight
  Error Bounds}.
\newblock In {\em Proc. FOCS}, 2014.

\bibitem{chaudhuri2011}
K.~Chaudhuri, C.~Monteleoni, and A.~D. Sarwate.
\newblock {Differentially Private Empirical Risk Minimization}.
\newblock {\em JMLR}, 2011.

\bibitem{Dwork2006}
C.~Dwork, F.~McSherry, K.~Nissim, and A.~Smith.
\newblock Calibrating noise to sensitivity in private data analysis.
\newblock In {\em Proceedings of the Third Conference on Theory of
  Cryptography}, 2006.

\bibitem{Dwork2014}
C.~Dwork and A.~Roth.
\newblock {\em {The Algorithmic Foundations of Differential Privacy}}.
\newblock Foundations and Trends in Theoretical Computer Science. Now, 2014.

\bibitem{frank1956algorithm}
M.~Frank and P.~Wolfe.
\newblock An algorithm for quadratic programming.
\newblock {\em Naval research logistics quarterly}, 1956.

\bibitem{fujishige2005submodular}
S.~Fujishige.
\newblock {\em {Submodular Functions and Optimization}}.
\newblock Elsevier, 2005.

\bibitem{jaggi13}
M.~Jaggi.
\newblock Revisiting {Frank-Wolfe}: Projection-free sparse convex optimization.
\newblock In {\em Proc. ICML}, 2013.

\bibitem{jain14}
P.~Jain and A.~Thakurta.
\newblock {(Near) Dimension Independent Risk Bounds for Differentially Private
  Learning}.
\newblock In {\em Proc. ICML}, 2014.

\bibitem{Kakade2009OnTD}
S.~Kakade and Shalev{-}Shwartz.
\newblock On the duality of strong convexity and strong smoothness : Learning
  applications and matrix regularization.
\newblock {\em Unpublished manuscript}, 2009.

\bibitem{Kasi2008}
S.~P. Kasiviswanathan and A.~D. Smith.
\newblock A note on differential privacy: Defining resistance to arbitrary side
  information.
\newblock {\em CoRR}, abs/0803.3946, 2008.

\bibitem{Kifer12}
D.~Kifer, A.~Smith, and A.~Thakurta.
\newblock {Private Convex Empirical Risk Minimization and High-dimensional
  Regression}.
\newblock In {\em Proc. ALT}, 2012.

\bibitem{lovasz1982submodular}
L.~Lov{\'a}sz.
\newblock {Submodular Functions and Convexity}.
\newblock {\em Mathematical Programming}, 1982.

\bibitem{obozinski2012convex}
G.~Obozinski and F.~Bach.
\newblock Convex relaxation for combinatorial penalties.
\newblock {\em arXiv preprint arXiv:1205.1240}, 2012.

\bibitem{rockafellar97}
R.~T. Rockafellar.
\newblock {\em {Convex Analysis}}.
\newblock Princeton University Press, 1997.

\bibitem{Shalev09}
S.~Shalev{-}Shwartz, O.~Shamir, N.~Srebro, and K.~Sridharan.
\newblock Stochastic convex optimization.
\newblock In {\em {COLT}}, 2009.

\bibitem{talwar2014}
K.~Talwar, A.~Thakurta, and L.~Zhang.
\newblock {Private Empirical Risk Minimization Beyond the Worst Case: The
  Effect of the Constraint Set Geometry}.
\newblock {\em arXiv preprint arXiv:1411.5417}, 2014.

\bibitem{Talwar2015}
K.~Talwar, A.~Thakurta, and L.~Zhang.
\newblock {Nearly-optimal Private LASSO}.
\newblock In {\em Proc. NIPS}, 2015.

\bibitem{Vapnik1995}
V.~N. Vapnik.
\newblock {\em The Nature of Statistical Learning Theory}.
\newblock Springer-Verlag, 1995.

\bibitem{Wang2017}
D.~Wang, M.~Ye, and J.~Xu.
\newblock {Differentially Private Empirical Risk Minimization Revisited: Faster
  and More General}.
\newblock In {\em Proc. NIPS}, 2017.

\bibitem{Wu2017}
X.~Wu, F.~Li, A.~Kumar, K.~Chaudhuri, S.~Jha, and J.~Naughton.
\newblock {Bolt-on Differential Privacy for Scalable Stochastic Gradient
  Descent-based Analytics}.
\newblock In {\em Proc. SIGMOD}, 2017.

\bibitem{zhang2017}
J.~Zhang, K.~Zheng, W.~Mou, and L.~Wang.
\newblock Efficient private {ERM} for smooth objectives.
\newblock {\em arXiv preprint arXiv:1703.09947}, 2017.

\end{thebibliography}

\appendix

\section*{Dual derivation to ERM}
\BEA
\min_{\theta \in \rb^p} \hat{\mathcal{L}}(\theta) + \frac{\lambda}{n} \Omega_{\infty}(\theta) 
& = & \min_{\theta \in \rb^p} \hat{\mathcal{L}}(\theta) + \frac{\lambda}{n} \sup_{s \in |P|(F)} s^{\top} \theta \nonumber \\
&   & \text{(as $\Omega_{\infty}$ is the support function of $|P|(F)$)} \nonumber \\
& = & \min_{\theta \in \rb^p} \sup_{s \in |P|(F)} \hat{\mathcal{L}}(\theta) +  \frac{\lambda}{n} s^{\top} \theta \nonumber \\ 
&   & \text{(due to strong duality)} \nonumber \\
& = & \sup_{s \in |P|(F)} \min_{\theta \in \rb^p} \hat{\mathcal{L}}(\theta) +  \frac{\lambda}{n} s^{\top} \theta  \nonumber \\
& = & \sup_{s \in |P|(F)} - \hat{\mathcal{L}}^*(-\frac{\lambda}{n}s),   \nonumber \\
\EEA
where $\hat{\mathcal{L}^*}$ is the Fenchel conjugate~\cite{rockafellar97} of $\mathcal{L}$. Let $K$ be the convex set represented by $\frac{\lambda}{n}|P|(F)$. Then the dual optimization problem is given by $\sup_{s \in K} - \hat{\mathcal{L}}^*(-s)$.

\end{document}